\documentclass[letterpaper]{article} 
\usepackage{aaai2026}  
\usepackage{times}  
\usepackage{helvet}  
\usepackage{courier}  
\usepackage[hyphens]{url}  
\usepackage{graphicx} 
\usepackage{natbib}  
\usepackage{caption} 
\usepackage{algorithm}
\usepackage{algorithmic}
\usepackage{amsmath, amssymb}
\usepackage{amsthm}
\newtheorem{theorem}{Theorem}

\usepackage{amsmath}
\usepackage{booktabs}
\usepackage{multirow}

\usepackage{newfloat}
\usepackage{listings}
\DeclareCaptionStyle{ruled}{labelfont=normalfont,labelsep=colon,strut=off} 
\floatstyle{ruled}
\newfloat{listing}{tb}{lst}{}
\floatname{listing}{Listing}
\newtheorem{prop}{Proposition}

\title{Post-Optimization Adaptive Rank Allocation for LoRA}
\author {
    Vishnuprasadh Kumaravelu\textsuperscript{\rm 1,2},
    Sunil Gupta\textsuperscript{\equalcontrib \rm 2},
    P. K. Srijith\textsuperscript{\equalcontrib \rm 1}
}
\affiliations {
    \textsuperscript{\rm 1}Indian Institute of Technology Hyderabad \ \ 
    \textsuperscript{\rm 2}Deakin University\\
    s224229262@deakin.edu.au, \ sunil.gupta@deakin.edu.au, \  srijith@cse.iith.ac.in
}

\begin{document}

\maketitle

\begin{abstract}
Exponential growth in the scale of modern foundation models has led to the widespread adoption of Low-Rank Adaptation (LoRA) as a parameter-efficient fine-tuning technique. However, standard LoRA implementations disregard the varying intrinsic dimensionality of model layers and enforce a uniform rank, leading to parameter redundancy. We propose Post-Optimization Adaptive Rank Allocation (PARA), a data-free compression method for LoRA that integrates seamlessly into existing fine-tuning pipelines. PARA leverages Singular Value Decomposition to prune LoRA ranks using a global threshold over singular values across all layers. This results in non-uniform rank allocation based on layer-wise spectral importance. As a post-hoc method, PARA circumvents the training modifications and resulting instabilities that dynamic architectures typically incur. We empirically demonstrate that PARA reduces parameter count by 75-90\% while preserving the predictive performance of the original, uncompressed LoRA across multiple vision and language benchmarks. Code will be published upon acceptance.
\end{abstract}


\section{Introduction}
Large pretrained models have induced a paradigm shift in the field of Artificial Intelligence. Trained on massive and diverse data corpora, these models demonstrate strong generalization and zero-shot capabilities \cite{gpt2}. As a result, practitioners increasingly opt to finetune these models on domain-specific datasets rather than training from scratch. As the scaling laws \cite{scalinglaws} continue to drive models past the trillion-parameter mark \cite{kimik2}, full-parameter finetuning has become computationally prohibitive. Moreover, finetuning models pretrained on internet-scale datasets on much smaller domain datasets is often sample-inefficient and prone to overfitting. This has led to the rise of Parameter Efficient Fine-Tuning (PEFT) methods \cite{peft}, which propose adapting such large models using a much smaller fraction of their parameters.

Among PEFT methods, Low-Rank Adaptation (LoRA) \cite{lora} has gained massive traction due to its simple architecture and strong empirical performance. LoRA is founded on the hypothesis that the updates induced by fine-tuning lie in a low-rank subspace. LoRA decomposes the weight update to each pretrained weight matrix as the product of two low-rank matrices, keeping the original pretrained parameters frozen and optimizing only the low-rank matrices during adaptation. Formally, for any linear layer $W \in \mathbb{R}^{d_1 \times d_2}$ in the pretrained model, LoRA reparameterizes its update $\Delta W$ as $\Delta W = BA$, with $B \in \mathbb{R}^{d_1 \times r}$ and $A \in \mathbb{R}^{r \times d_2}$ being the low rank matrices and rank, $r \ll \min(d_1,d_2)$.

The choice of rank dictates both the parameter budget and adaptation capacity. While LoRA can match full fine-tuning performance given a sufficiently large rank \cite{lorawithoutregret}, empirical evidence shows that fine-tuning performance is highly sensitive to this choice \cite{dylora}. Excessive rank yields diminishing returns and can lead to overfitting, while insufficient rank limits expressivity. Despite its significance, selecting the right rank remains a heuristic process requiring computationally expensive hyperparameter sweeps that undermines the efficiency gains that motivate PEFT methods in the first place.

This challenge becomes particularly acute in production environments. While it is common for simple single-task deployments to merge LoRA with the base model to eliminate additional inference latency, modern ML systems increasingly operate in multi-tenant regimes where a frozen backbone supports thousands of distinct task-specific adapters \cite{phlora}. In such settings, adapters are dynamically swapped into GPU memory on demand, and adapter size directly bottlenecks system throughput through memory capacity and bandwidth constraints. The rank selection problem thus manifests as both a training-time optimization challenge and a deployment-time efficiency bottleneck.

Beyond rank allocation, the optimal location for LoRA updates remains an open question. Early work restricted adaptation to attention weights \cite{lora}, but QLoRA \cite{qlora} later argued that including MLP layers is critical for performance. Conversely, \citeauthor{loralearnsless} \cite{loralearnsless} found adaptation of attention layers redundant when MLPs are tuned. Methods such as AdaLoRA \cite{adalora}, SoRA \cite{sora}, and DoRA \cite{dora} propose to automate this selection by treating rank and layer selection as a joint optimization problem and dynamically pruning ranks during training. However, these adaptive approaches introduce significant complexity, relying on architectural modifications, expensive parameter-importance calculations, and intricate pruning schedules that can destabilize training if improperly tuned. Consequently, the burden shifts from selecting a fixed rank to tuning a complex suite of pruning schedules and regularization constants, effectively increasing the hyperparameter search space rather than reducing it.

To address these challenges, we propose Post-Optimization Adaptive Rank Allocation (PARA). PARA is a data-free compression framework for pre-trained LoRA adapters. Unlike dynamic methods, PARA imposes no architectural changes or training overhead. For PARA, we recommend training all transformer layers with standard LoRA at a sufficiently high rank to ensure capacity during optimization, relying on the post-hoc compression to discard extra ranks and redundant adapter layers. Upon training completion, PARA applies Singular Value Decomposition (SVD) to each adapter's weight matrix. This decomposition separates the learned update $\Delta W$ into singular vectors, which encode the transformation directions, and singular values, which quantify the magnitude of those transformations. Consequently, the singular values serve as a direct proxy for the importance of each rank dimension. A single global threshold is then applied to the collective singular values across the entire model. This threshold is determined by the practitioner's choice of either a target average rank (given the ranks vary across layers) or a relative energy retention ratio (the fraction of the total Frobenius norm preserved). This mechanism automatically allocates higher ranks to layers with a greater concentration of globally large singular values while aggressively pruning, or even entirely discarding, layers with negligible singular values. 

Crucially, PARA's adaptive rank allocation consistently outperforms LoRA adapters natively trained at the equivalent parameter budget. Standard LoRA enforces a uniform rank constraint, bottlenecking layers that require higher expressivity to model complex features, while simultaneously wasting parameters on layers that contribute little to the task. By decoupling the training rank from the final inference rank, PARA circumvents this limitation. Training in a high-rank regime allows the adapter to explore a larger subspace of solutions, leading to better convergence. Based on spectral importance, PARA then redistributes the parameter budget, yielding an adapter that is more accurate than its uniformly trained counterpart. The result is a non-uniform rank distribution (Figure~\ref{fig:heatmap}) derived analytically from the learned weights themselves, preserving the training stability and architectural simplicity of standard LoRA.

\begin{figure}
    \centering
    \includegraphics[width=0.99\linewidth]{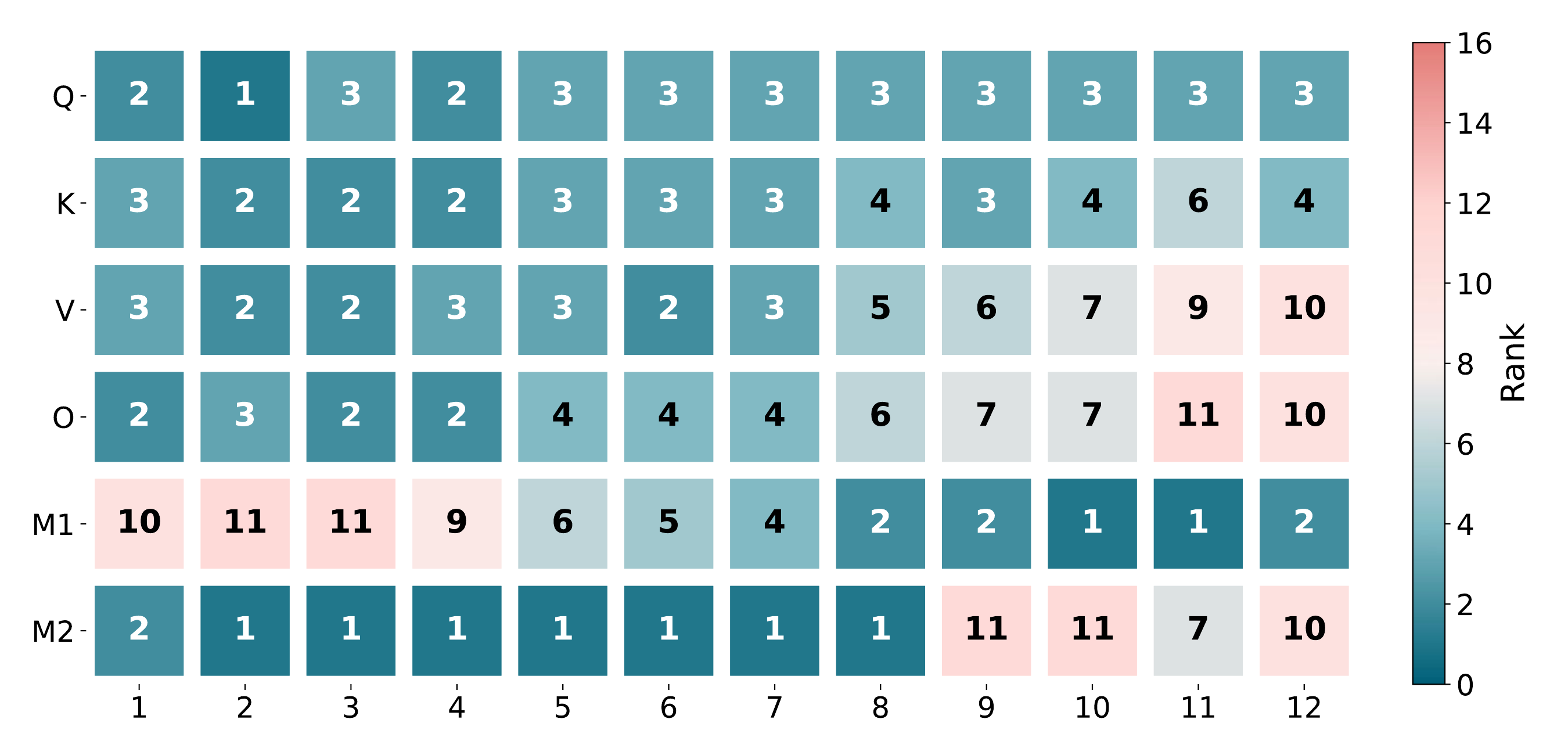}
    \caption{Rank distribution across layer types and depth on a LoRA trained at rank 16 on the Food 101 Image Classification task and compressed by PARA to an average rank of 4. PARA automatically allocates ranks based on spectral importance.}
    \label{fig:heatmap}
\end{figure}

PARA's use of singular values as an importance metric fundamentally contrasts with training-time pruning methods like AdaLoRA that rely on expensive gradient-based parameter importance calculations. However, a naive application of SVD to large Transformer weight matrices can be computationally expensive as well. To address this, PARA exploits the intrinsic low-rank structure of the adapters and performs SVD via QR decomposition. This strategy allows us to compute singular values without ever materializing the full weight matrix, reducing the computational complexity. To the best of our knowledge, PARA is the first framework to integrate this numerical optimization into the LoRA ecosystem. The result is a highly efficient process that yields compact adapters with minimal memory footprints, significantly lowering VRAM consumption and enabling higher concurrency in multi-tenant serving scenarios. PARA's efficiency enables a "Train First, Tune Later" workflow, where rank selection is decoupled from training and delegated to the deployment phase. This also enables a versatile ``one-to-many'' deployment strategy. Similar to DyLoRA~\cite{dylora}, a single high-rank LoRA parent can spawn a family of PARA-compressed child adapters at varying sizes to satisfy diverse latency constraints. However, unlike DyLoRA's uniform rank truncation, PARA's adaptive-rank adapters maximize performance for every target budget. Across various image classification, natural language understanding, and generation benchmarks, PARA achieves 75-90\% parameter reduction with less than 1\% accuracy degradation in most cases. PARA consistently outperforms existing adaptive rank methods such as AdaLoRA~\cite{adalora}, SoRA~\cite{sora}, and DoRA~\cite{dora}, as well as multi-rank methods such as DyLoRA~\cite{dylora} across benchmarks.

Our contributions can be summarized as follows:
\begin{enumerate}
    \item We introduce PARA, a data-free compression framework that automatically optimizes rank distribution using the singular value spectrum of learned updates.
    \item We establish a ``Train First, Tune Later'' paradigm that decouples training capacity from inference constraints to maximize both optimization stability and deployment efficiency.
    \item We enable versatile ``one-to-many'' deployment, allowing a single parent model to spawn a family of adaptive adapters for varying latency budgets.
    \item We exploit the intrinsic low-rank structure of LoRA adapters via QR-based SVD, enabling efficient spectral analysis without materializing full weight matrices.
    \item Our method outperforms baselines, achieving 75-90\% parameter reduction with negligible accuracy loss across diverse vision and language benchmarks.
\end{enumerate}

\section{Related Work}
\subsection{Adaptive Rank LoRA Variants}
The rigidity of LoRA's uniform rank assignment has spurred a variety of adaptive rank methods. \textbf{AdaLoRA} \cite{adalora} sets a universal parameter budget and leverages parameter importance to adaptively determine ranks for different layers. Based on a pruning schedule, AdaLoRA begins with a uniform rank distribution and progressively prunes less informative ranks until a target budget is reached, resulting in varying rank allocations. Crucially, AdaLoRA makes architectural changes to LoRA and adds additional regularization terms to the loss function. These modifications, along with expensive periodic parameter importance calculations and tuning the pruning schedule, complicate training, requiring multiple hyperparameter decisions and validations. \textbf{SoRA} \cite{sora} modifies LoRA by introducing a trainable gating vector between the LoRA matrices to progressively prune rank. It employs a proximal gradient method with $L_1$ regularization to update the gate, effectively performing soft-thresholding to promote sparsity during training. Post-training, zeroed-out ranks are pruned to yield a LoRA adapter with variable ranks. \textbf{DoRA} \cite{dora} reparameterizes LoRA into sums of single-rank components, dynamically pruning them based on their contribution to the total Frobenius norm of the update. To prevent instability, DoRA employs an additional regularization term that penalizes the variance of elements within components, to ensure a balanced magnitude distribution. SoRA and DoRA both incur additional training overhead in the form of architectural modifications, importance calculations, additional regularization terms, or bespoke optimization strategies. Unlike the methods discussed so far, which are training-time modifications to LoRA, \textbf{GoRA} \cite{gora} is an initialization framework. GoRA computes sensitivity-based parameter importance for each backbone layer using accumulated gradients from a subset of training samples and allocates ranks accordingly. In addition, it initializes matrix $A$ randomly as usual, but initializes matrix $B$ using a formula involving the pseudo-inverse of $A$ and the accumulated gradients to approximate the initial gradient update. Consequently, GoRA sits philosophically opposite to PARA. While GoRA is a data-driven initialization strategy designed to optimize LoRA's initialization before fine-tuning, PARA is a data-free post-optimization technique designed to compress LoRA after fine-tuning.

\subsection{Simultaneous Multi-Rank Training}

\textbf{DyLoRA} \cite{dylora} extends LoRA by training for a range of ranks simultaneously rather than a single fixed rank. Following Nested Dropout \cite{nested-dropout}, DyLoRA randomly samples a rank during training, truncates the LoRA matrices, and updates the parameters. This forces the adapters to learn ordered representations, ensuring critical information is concentrated in the lower ranks. At inference, DyLoRA can be deployed at any rank within the trained range, enabling dynamic adaptation to hardware constraints without retraining. PARA brings similar post-training dynamic rank adjustment benefits to standard LoRA without any bespoke training adjustments, solely relying on spectral importance to enable ordered pruning. 

\subsection{Spectral Decomposition in LoRA}
Recent works have explored leveraging Singular Value Decomposition in the context of LoRA to improve initialization. \textbf{PiSSA} \cite{pissa} initializes LoRA matrices with the principal singular components of the pretrained weight matrix, effectively allowing the adapter to optimize the principal subspace directly. On the contrary, \textbf{MiLoRA} \cite{milora} initializes adapters using minor components in order to mitigate forgetting of pretrained knowledge. While PiSSA and MiLoRA use SVD on pretrained matrices for initialization to aid convergence, PARA uses SVD on fine-tuned LoRA matrices to aid deployment efficiency. Furthermore, PARA is complementary to these approaches; one could theoretically initialize with PiSSA to accelerate training, and subsequently apply PARA to prune the resulting adapters for efficient inference. \textbf{PhLoRA} \cite{phlora} uses SVD to extract a LoRA adapter from a fine-tuned model by performing truncated SVD on the weight update $\Delta W$. While PhLoRA and PARA both use SVD as a compression technique, PhLoRA focuses on full model weights to generate regular LoRAs whereas PARA's focus is on redistributing the rank budget of LoRAs in order to maximize efficiency.

\section{Methodology} \label{sec:methodology}

Let $\mathcal{M}_\Theta$ be a pretrained transformer parameterized by weights $\Theta = \{W^y_i \ | \ y \in Y , i \in \{1,2,\dots,N\}\}$, where $Y = \{q, k, v, o, m_1, m_2\}$ denotes the layer types and $N$ is the number of transformer layers. Each transformer layer comprises the attention matrices: $q$ (query), $k$ (key), $v$ (value) and $o$ (out-projection), and the two MLP matrices $m_1$ and $m_2$ respectively. The biases are ignored for brevity. Assuming matrices of all types are adapted, let $\Phi = \{\phi^y_i \ | \ y \in Y , i \in \{1,2,\dots,N\}\}$ be the set of all LoRA adapters, where $\phi^y_i = B^y_i A^y_i$. We define the total rank budget of the model as $\mathcal{B} = \sum_{y,i} rank(\phi^y_i)$. Given that the initial rank $r$ is fixed across all matrices, $\mathcal{B}_{init} = r \cdot N \cdot |Y|$. Let $\mathcal{B}_{tgt} < B_{init}$ be the target budget after compression. As PARA compression results in variable ranks across matrices, we define the average target rank after compression $\bar{r} = \frac{{B}_{tgt}}{N|Y|}$.

\subsection{Post-Optimization Adaptive Rank Allocation}

PARA seeks to compress LoRA adapters by identifying and pruning redundant singular values and corresponding singular vectors globally across the network. We first consider the Singular Value Decomposition of each LoRA $\phi^y_i$:
\begin{equation}
    \phi^y_i = B^y_i A^y_i = U^y_i \Sigma^y_i (V^y_i)^T
\end{equation}
Since $\phi^y_i$ is the product of two rank-$r$ matrices, it has at most rank $r$. Thus, we compute the compact SVD where $\Sigma^y_i = diag(\sigma^y_{i,1}, \dots, \sigma^y_{i,r})$ contains the $r$ singular values and $U^y_i \in \mathbb{R}^{d_1 \times r}$ and $V^y_i \in \mathbb{R}^{d_2 \times r}$ contain the corresponding left and right singular vectors. 

Let $\mathcal{E} = \bigcup_{y,i}\{\sigma^y_{i,j}\}_{j=1}^r$ be the global set of all singular values across all adapters. To satisfy the target budget $\mathcal{B}_{tgt}$, we determine a global threshold $\tau$, such that exactly $\mathcal{B}_{tgt}$ singular values in $\mathcal{E}$ are greater than or equal to $\tau$. We construct the pruned diagonal matrix $\hat{\Sigma}^y_i$ by masking singular values below the global threhold $\tau$:
\begin{equation}
(\hat{\Sigma}^y_i)_{jj} = (\Sigma^y_i)_{jj} \cdot \mathbb{I}[(\Sigma^y_i)_{jj} \geq \tau] 
\end{equation}
where $\mathbb{I}(\cdot)$ is the identity function. We then reconstruct the LoRA adapters using the pruned singular values. To maintain the LoRA structure, we distribute the singular values symmetrically:
\begin{equation}
    U^y_i \hat{\Sigma}^y_i (V^y_i)^T = [U^y_i \sqrt{\hat{\Sigma}^y_i}] [\sqrt{\hat{\Sigma}^y_i} (V^y_i)^T] = \hat{B}^y_i \hat{A}^y_i = \hat{\phi}^y_i
\end{equation}
This results in adapters with heterogeneous ranks $0 \leq \hat{r}^y_i \leq r$, allocating more budget to layers with higher spectral energy and effectively removing layers where contributions are negligible

\subsection{SVD in the LoRA Subspace via QR Decomposition}

Directly computing the SVD of $\phi^y_i$ incurs a computational complexity of $O(d_1 d_2^2)$ where $d_1 \geq d_2$. For large pretrained transformers, this is prohibitively expensive. However, the architecture of LoRA enforces a strong inductive bias where the optimization is confined to a low-rank subspace $r << d_1,d_2$. Rather than multiplying $B^y_i$ and $A^y_i$ to materialize $\phi^y_i$, which by itself costs $O(d_1d_2r)$, and then performing an ambient-space decomposition, PARA exploits the low-rank structure directly. Following \cite{qr-svd}, we compute SVD via QR decomposition of the LoRA matrices without requiring the full matrix:
\begin{equation}
    B^y_i = Q_{B^y_i} R_{B^y_i}
    \label{eq:bqr}
\end{equation}
\begin{equation}
    (A^y_i)^T = Q_{A^y_i} R_{A^y_i} \implies A^y_i = R_{A^y_i}^T Q_{A^y_i}^T
    \label{eq:aqr}
\end{equation}
where $Q_{B^y_i} \in \mathbb{R}^{d_1 \times r}$ and $Q_{A^y_i} \in \mathbb{R}^{d_2 \times r}$ have orthonormal columns spanning the column space of B and the row space of A respectively and, $R_{B^y_i} \in \mathbb{R}^{r \times r}$ and $R_{A^y_i} \in \mathbb{R}^{r \times r}$ are upper triangular matrices. Substituting the QR forms in $\phi^y_i = B^y_i A^y_i$, we have:
\begin{equation} \label{eq:qr-in-lora}
    \phi^y_i = Q_{B^y_i} [R_{B^y_i} R^T_{A^y_i}] Q^T_{A^y_i} \\
\end{equation}
We define the interaction matrix $M^y_i = R_{B^y_i} R^T_{A^y_i}$. $M \in \mathbb{R}^{r \times r}$ encapsulates the interaction between the input and output subspaces entirely within the latent dimension $r$. We then compute the SVD of $M$, thus discarding the null space before decomposition rather than during it:
\begin{equation} \label{eq:svd-m}
    M^y_i = \tilde{U}^y_i \Sigma^y_i (\tilde{V}^y_i)^T
\end{equation}
Substituting Equation \ref{eq:svd-m} back in Equation \ref{eq:qr-in-lora}
\begin{equation}
\begin{split}
    \phi^y_i &= Q_{B^y_i} [\tilde{U}^y_i \Sigma^y_i (\tilde{V}^y_i)^T] Q^T_{A^y_i} \\
    &= [Q_{B^y_i} \tilde{U}^y_i] \Sigma^y_i [(\tilde{V}^y_i)^T Q^T_{A^y_i}] \\
    &= [Q_{B^y_i} \tilde{U}^y_i] \Sigma^y_i [Q_{A^y_i} \tilde{V}^y_i]^T \\
    &= U^y_i \Sigma^y_i (V^y_i)^T
\end{split}
\label{eq:svd_derivation}
\end{equation}
This approach is mathematically equivalent to the full SVD. 

\begin{prop}[Orthogonality of Derived Bases]
The matrices $U^y_i = Q_{B^y_i} \tilde{U}^y_i$ and $V^y_i = Q_{A^y_i} \tilde{V}^y_i$ form valid orthonormal bases for the Singular Value Decomposition of $\phi^y_i = B^y_i A^y_i$.
\end{prop}

\begin{proof}
We verify the orthonormality of the left singular vectors $U^y_i$:
\begin{equation}
\begin{aligned}
    (U^y_i)^T U^y_i &= (Q_{B^y_i} \tilde{U}^y_i)^T (Q_{B^y_i} \tilde{U}^y_i) \\
    &= (\tilde{U}^y_i)^T (Q_{B^y_i}^T Q_{B^y_i}) \tilde{U}^y_i \\
    &= (\tilde{U}^y_i)^T \mathbb{I} \tilde{U}^y_i \quad \text{(since } Q_{B^y_i} \text{ has orthonormal columns)} \\
    &= (\tilde{U}^y_i)^T \tilde{U}^y_i \\
    &= \mathbb{I} \quad \text{(since } \tilde{U}^y_i \text{ is unitary)}
\end{aligned}
\end{equation}
The derivation for $V^y_i$ follows symmetrically. Thus, the decomposition $U^y_i \Sigma^y_i (V^y_i)^T$ satisfies all conditions of the SVD. Since the singular values $\Sigma^y_i$ are intrinsic to the operator $\phi^y_i$, they are invariant to the method of computation.
\end{proof}

Crucially, the computational cost is now determined by the QR decomposition of the $A$ and $B$ matrices, followed by the SVD of the interaction matrix $M$: $O(d_2r^2) + O(d_1r^2) + O(r^3)$. Thus, the computational costs are reduced to $O((d_1 + d_2)r^2 + r^3)$, where $r \ll d_1, d_2$. Given that $d_1$ and $d_2$ are typically in the thousands and $r \in [8,64]$, this optimization reduces the decomposition overhead by orders of magnitude, making PARA feasible for very large transformer models. We summarize the complete procedure for Post-Optimization Adaptive Rank Allocation in Algorithm \ref{alg:PARA}.

\begin{figure}[t]
    \centering
    \includegraphics[width=0.99\linewidth]{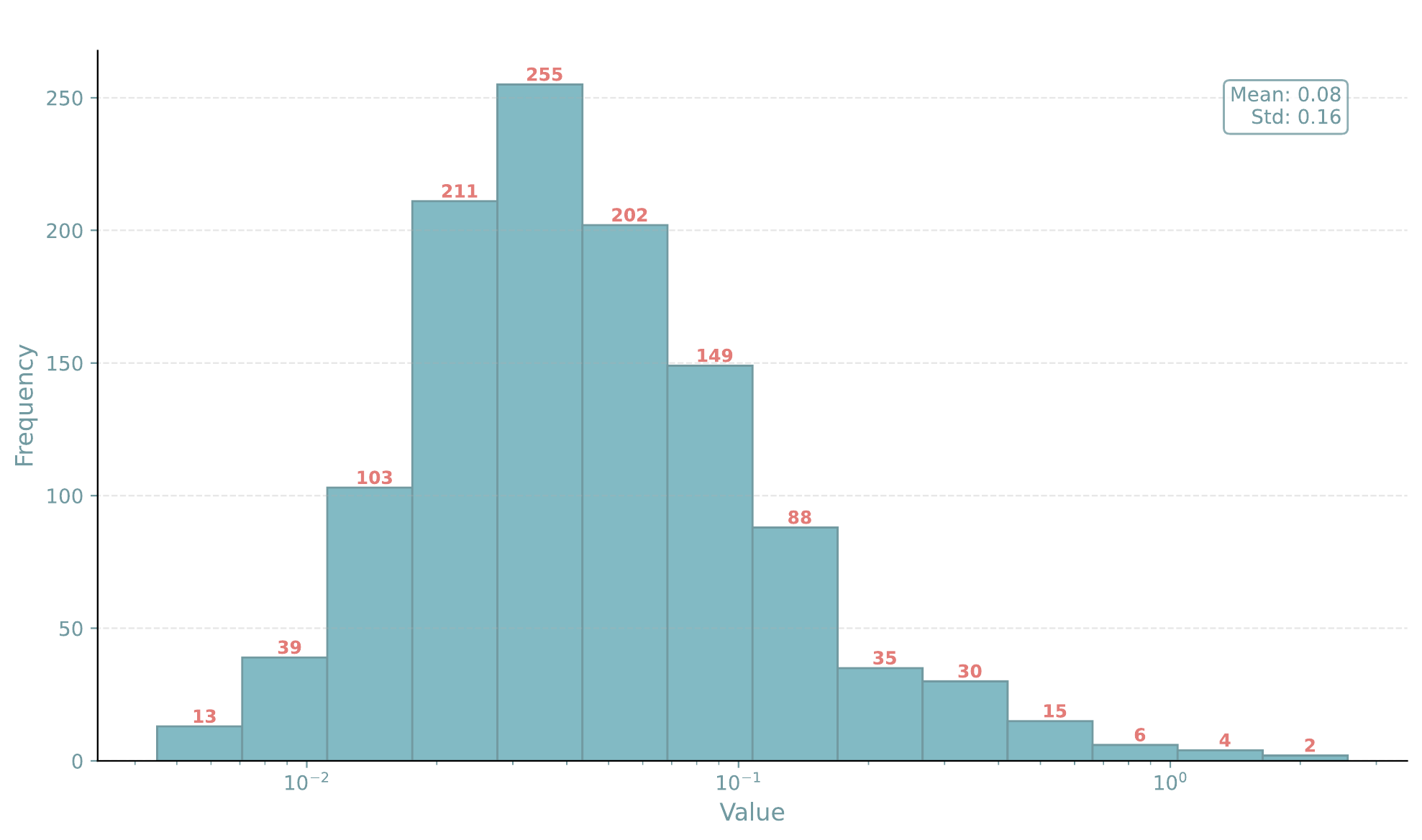}
    \caption{Distribution of singular values from a LoRA of rank 16 trained on the Food-101 image classification dataset. Similar plots from other datasets are presented in Figure~\ref{fig:appendix-hist}.}
    \label{fig:singular-values-dist}
\end{figure}

\subsection{Justification}
\label{sec:theory}

Analyzing the distribution of singular values of Low Rank Adapters (Fig.~\ref{fig:singular-values-dist}), we observe that most singular values are near zero, with a pronounced long tail. The spectral energy of the update is dominated by a handful of singular values that occupy the long tail of the distribution, implying that the effective rank of the update is much lower than the LoRA rank. This indicates that the LoRA has learned a few strong directions, while the remaining components contribute marginally and may represent memorization or noise. 

The use of singular values as an importance metric is grounded in the Eckart-Young-Mirsky theorem \citep{eym1, eym2}. The theorem provides the optimal low-rank approximation of a matrix. It states that the best rank-$k$ approximation of a matrix $A$, in both the Frobenius norm or any unitarily invariant norm, is obtained by truncating the Singular Value Decomposition (SVD) of $A$ after the $k$-th largest singular value. 

\begin{theorem}[Eckart-Young-Mirsky \cite{eym1, eym2}]
Let $\phi \in \mathbb{R}^{d_1 \times d_2}$ have singular value decomposition $\phi = U \Sigma V^\top$ with singular values $\sigma_1 \geq \sigma_2 \geq \cdots \geq \sigma_r > 0$. For any target rank $k < r$, the truncated SVD
\begin{equation}
    \phi_k = \sum_{i=1}^{k} \sigma_i u_i v_i^\top
\end{equation}
is the unique minimizer of $\|\phi - \hat{\phi}\|_F$ over all matrices $\hat{\phi}$ with $\text{rank}(\hat{\phi}) \leq k$. Moreover, the approximation error is exactly:
\begin{equation}
    \|\phi - \phi_k\|_F = \sqrt{\sum_{i=k+1}^{r} \sigma_i^2}
\end{equation}
\end{theorem}

Each singular value $\sigma_i$ contributes exactly $\sigma_i^2$ to the squared Frobenius norm of the update, which measures the total magnitude of the weight modification. Pruning small singular values thus removes components that contribute minimally to the adaptation signal.




With PARA, compression is performed optimally in the Frobenius norm sense, and rank is allocated to layers in proportion to their spectral energy. Layers with concentrated spectral energy retain higher ranks, while layers with diffuse or negligible spectra are aggressively pruned or discarded entirely.

\subsection{Threshold Selection Policies}

While the global threshold $\tau$ dictates the compression rate, selecting its optimal value depends on the deployment constraints. We propose two distinct policies for determining $\tau$.

\subsubsection{$\gamma$-PARA} This policy mirrors that of other adaptive rank LoRA frameworks such as AdaLoRA, DoRA, etc where the practitioner is required to make a decision on the target rank based on deployment constraints. The optimal target average rank can be attained through a validation set guided search. Unlike LoRA's rank selection which requires mutliple training runs for different ranks, PARA is training-free and can generate LoRAs across a range of ranks much more efficiently. Given a rank preservation ratio $\gamma \in (0,1]$, we obtain the target average rank $\bar{r} = \gamma \cdot r$, and subsequently, the target parameter budget $\mathcal{B}_{tgt} = \bar{r} \cdot N \cdot |Y|$. The threshold is then calculated as follows:
\begin{equation} \label{eqn:tau-r}
    \tau_{\gamma} = \mathcal{E}^\downarrow[\mathcal{B}_{tgt}]
\end{equation}
where $\tau_{\gamma}$ is the $\mathcal{B}_{tgt}$-th element of the list of singular values sorted in descending order, $\mathcal{E}^\downarrow$.

\subsubsection{$\epsilon$-PARA} This policy relies solely on the intrinsic spectral properties of the adapters to define the global threshold. We define the total spectral energy as the sum of squared singular values from all adapter layers:
\begin{equation}
E_{total} = \sum_{i=1}^N \sum_{y \in Y} \sum_{\sigma \in \Sigma^y_i} \sigma^2
\end{equation}
Given a target preservation ratio $\epsilon \in (0, 1]$ (e.g., $\epsilon = 0.99$), we seek the maximum threshold $\tau$ such that the retained energy meets the target:
\begin{equation} \label{eqn:tau-e}
\tau_\epsilon = \max { \tau \mid \sum_{i=1}^N \sum_{y \in Y} \sum_{\sigma \in \Sigma^y_i} \sigma^2 \cdot \mathbb{I}(\sigma \ge \tau) \ge \epsilon \cdot E_{total} }
\end{equation}
This policy guarantees that the compressed model retains $\epsilon$-fraction of the update's information content in the Frobenius norm sense.



\begin{algorithm}[tb]
   \caption{Post-Optimization Adaptive Rank Allocation (PARA)}
   \label{alg:PARA}
\begin{algorithmic}
   \STATE {\bfseries Input:} LoRA adapters $\Phi$, \\ Rank Preservation Ratio $\gamma$ or Target Energy Preservation ratio $\epsilon$
   \STATE {\bfseries Output:} Compressed LoRA adapters $\hat{\Phi}$
   \STATE {\bfseries Init:} Singular value pool $\mathcal{E} \leftarrow \emptyset$ and cache $\mathcal{C} \leftarrow \emptyset$
   
   \STATE \textbf{\textit{Phase 1: Decomposition}}
   \FOR{each layer $i \in \{1, \dots, N\}$}
        \FOR{each layer type $y \in Y$}
           \STATE $Q_{B^y_i}, R_{B^y_i} \leftarrow \text{QR}(B^y_i)$
           \STATE $Q_{A^y_i}, R_{A^y_i} \leftarrow \text{QR}((A^y_i)^T)$
           \STATE $M^y_i \leftarrow R_{B^y_i} R_{A^y_i}^T$ 
           \STATE $\tilde{U}^y_i, \Sigma^y_i, (\tilde{V}^y_i)^T \leftarrow \text{SVD}(M^y_i)$
            \STATE $U^y_i \leftarrow Q_{B^y_i} \tilde{U}^y_i$
            \STATE $V^y_i \leftarrow Q_{A^y_i} \tilde{V}^y_i$
           \STATE $\mathcal{E} \leftarrow \mathcal{E} \cup \text{diag}(\Sigma^y_i)$ \COMMENT{Collect singular values}
           \STATE Store $(\tilde{U}^y_i, \tilde{V}^y_i, \Sigma^y_i)$ in $\mathcal{C}$
        \ENDFOR
   \ENDFOR
   
    \STATE \textbf{\textit{Phase 2: Pruning \& Reconstruction}}
   \STATE Calculate $\tau$ as given by Eqn.~\ref{eqn:tau-r} or Eqn.~\ref{eqn:tau-e} per policy 
      \FOR{each layer $i \in \{1, \dots, N\}$}
        \FOR{each layer type $y \in Y$}
           \STATE Retrieve $(\tilde{U}^y_i, \tilde{V}^y_i, \Sigma^y_i)$ from $\mathcal{C}$
           \STATE Construct $\hat{\Sigma}^y_i$ such that $(\hat{\Sigma}^y_i)_{jj} = (\Sigma^y_i)_{jj} \cdot \mathbb{I}((\Sigma^y_i)_{jj} \geq \tau)$
           \STATE $\hat{B}^y_i \leftarrow U^y_i \sqrt{\hat{\Sigma}^y_i}$
           \STATE $\hat{A}^y_i \leftarrow \sqrt{\hat{\Sigma}^y_i} (V^y_i)^T$
           \STATE $\hat{\Phi} \leftarrow \hat{\Phi} \cup \{ (\hat{B}^y_i, \hat{A}^y_i) \}$
        \ENDFOR
   \ENDFOR
   
   \STATE \textbf{return} $\hat{\Phi}$
\end{algorithmic}
\end{algorithm}


\section{Experiments}

We benchmark PARA under diverse settings such as standard image classification, natural language understanding, multi-rank deployment and model merging. In all our experiments, unless mentioned explicitly otherwise, we follow the $\gamma$-PARA policy with $\gamma = 0.25$, resulting in a compressed LoRA one fourth the original size. This is to maintain uniformity in number of parameters with our baselines. We employ the standard LoRA \cite{lora} as our first baseline. However, unlike the original implementation that only adapts attention matrices, our implementation applies LoRA adapters to all layers for maximum performance. In addition, we consider the following Adaptive Rank LoRA variants as baselines: AdaLoRA \cite{adalora}, SoRA \cite{sora}, DoRA \cite{dora}, and GoRA \cite{gora}. For the training-time adaptive rank frameworks, namely AdaLoRA, SoRA and DoRA, we start training at $r_{init}$ and progressively prune to $r_{tgt}$. For GoRA, which is an initialization method, we allocate a parameter budget equivalent to an average rank $r_{tgt}$. All experiments were conducted on a single Nvidia H100 GPU. Pruning schedules and other baseline configuration details are presented in the Appendix.



\subsection{Image Classification} \label{sec:img-class}
We benchmark PARA and the baselines on standard image classification datasets such as CIFAR-10 \cite{cifar}, CIFAR-100 \cite{cifar}, Eurosat \cite{eurosat1, eurosat2}, Oxford Flowers \cite{flowers}, Oxford-IIIT Pet \cite{pets}, Stanford Cars \cite{cars}, and Food-101 \cite{food}. For vision tasks, we adopt SigLIP2 Base's vision encoder \cite{siglip2} as the frozen backbone and set $r_{init} = 16$ and $r_{tgt} = 4$. Our empirical results (Tab.~\ref{tab:image-classification}) show that PARA outperforms LoRA and the adaptive LoRA variants AdaLoRA, SoRA and GoRA across all datasets. PARA performs better than DoRA in all but two datasets, namely Oxford-IIIT pet and Oxford Flowers, where it finishes second. Fig~\ref{fig:cars-dual} demonstrates $\epsilon$-PARA on the Stanford Cars Dataset. The rank exponentially decays with reducing energy but performance remains stable until a certain point, after which it drops. Interestingly, we see that the higher rank compressed adapters perform slightly better than the parent adapter. This could be attributed to the spectral pruning eliminating noise, resulting in a clearer signal. We observe consistent behaviors across other image classification datasets as well (Fig.~\ref{fig:appendix-dual}).

Conversely, we plot Fig. \ref{fig:cars-contrast}. The energy–rank curve shows that most performance is retained while aggressively reducing rank when discarding low-energy (trailing) singular directions, whereas the reverse ablation demonstrates that removing a small number of high-energy singular directions catastrophically degrades accuracy; together, these results indicate that LoRA updates are dominated by a few task-critical singular components, making trailing-singular-value pruning a reliable and principled compression strategy.

\begin{figure}
    \centering
    \includegraphics[width=0.85\linewidth]{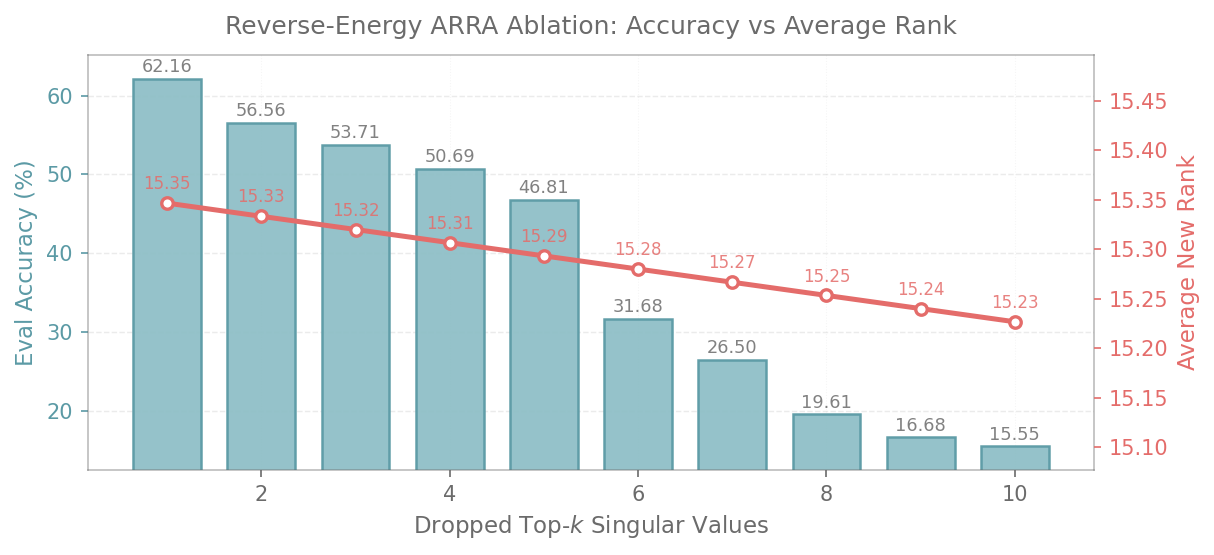}
    \caption{Plot denoting Accuracy (blue bars) and Rank (red dots) at various compression levels where top K singular values are dropped (Uniform $\epsilon$-PARA compression not feasible due to high concentration of energy in principal singular values) of LoRA of rank 16 trained on the Stanford Cars Image Classification dataset.}
    \label{fig:cars-contrast}
\end{figure}

\begin{table*} 
\centering
\caption{Accuracy Results of PARA and baselines on Image Classification Tasks. PARA outperforms baselines and scores highest (bold) or second-highest across datasets.}
\label{tab:image-classification}
\begin{tabular}{@{}lcccccccc@{}}
\toprule
\textbf{Method} & \multicolumn{1}{l}{\textbf{CIFAR 10}} & \multicolumn{1}{l}{\textbf{CIFAR 100}} & \multicolumn{1}{l}{\textbf{Pets}} & \multicolumn{1}{l}{\textbf{Flowers}} & \multicolumn{1}{l}{\textbf{Cars}} & \multicolumn{1}{l}{\textbf{Food}} & \multicolumn{1}{l}{\textbf{Eurosat}} & \multicolumn{1}{l}{\textbf{Average}} \\ \midrule
LoRA            & 97.53                                 & 68.11                                  & 87.08                             & 82.18                                & 78.55                             & 79.43                             & 97.22                                & 84.30                                 \\
AdaLoRA         & 96.09                                 & 50.8                                   & 88.5                              & 84.14                                & 54.58                             & 68.43                             & 96.22                                & 76.96                          \\
SoRA            & 96.24                                     & 60.12                                      & 82.17                                 & 83.41                                    & 64.38                                 & 72.19                                 & 96.22                                    & 79.25                                    \\
DoRA            & 96.51                                 & 64.48                                  & \textbf{90.49}                    & \textbf{87.87}                       & 66.63                             & 74.49                             & 96.30                                 & 82.39                          \\
GoRA            & 96.91                                 & 67.76                                  & 83.35                             & 79.8                                 & 71.10                              & 80.66                             & 97.56                                & 82.45                          \\
PARA            & \textbf{97.89}                        & \textbf{79.08}                         & 89.97                             & 86.03                                & \textbf{84.58}                    & \textbf{86.4}                     & \textbf{97.85}                       & \textbf{88.83}                          \\ \bottomrule
\end{tabular}
\end{table*}

\subsection{Natural Language Understanding} \label{sec:nlu}
For natural language understanding tasks, we benchmark on the GLUE tasks \cite{glue} MNLI \cite{mnli}, SST-2 \cite{sst2}, CoLA \cite{cola}, QNLI \cite{qnli}, and MRPC \cite{mrpc}. We use RoBERTa Base \cite{roberta} as the frozen backbone model. Following Sec~\ref{sec:img-class}, we set $r_{init}=16$ and $r_{tgt} = 4$. Empirical results from the NLU tasks (Tab.\ref{tab:natural-language-understanding}) also paint a picture similar to that of Image Classification, where PARA outperforms LoRA and its adaptive rank variants. $\epsilon$-PARA behaves in a manner similar to Image Classification (Fig.~\ref{fig:appendix-dual}) where the rank decays with reducing spectral energy, enabling high compression ratios whilst preserving performance. 

\begin{figure}
    \centering
    \includegraphics[width=0.85\linewidth]{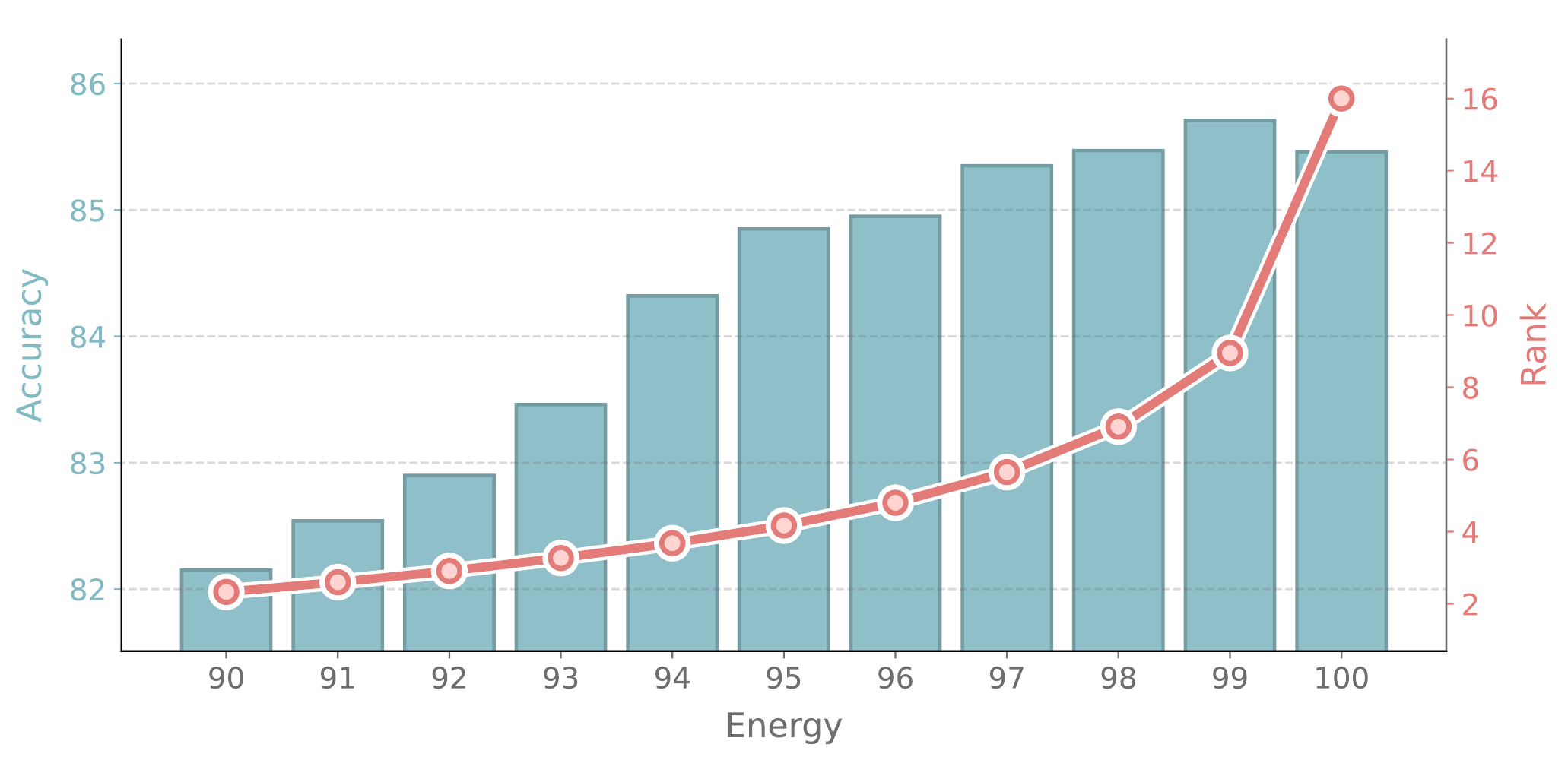}
    \caption{Plot denoting Accuracy (blue bars) and Rank (red dots) at various compression levels as a percentage of Total Spectral Energy of LoRA of rank 16 trained on the Stanford Cars Image Classification dataset.}
    \label{fig:cars-dual}
\end{figure}

\begin{table*}[]
\centering
\caption{Accuracy Results of PARA and baselines with RoBERTA-Base on Natural Language Understanding Tasks. PARA outperforms baselines and scores highest (bold) across datasets.}
\label{tab:natural-language-understanding}
\begin{tabular}{@{}lccccccc@{}}
\toprule
\textbf{Method} & \multicolumn{1}{l} {\textbf{MNLI(M)}} & \multicolumn{1}{l}{\textbf{MNLI(U)}} & \multicolumn{1}{l}{\textbf{SST-2}} & \multicolumn{1}{l}{\textbf{CoLA}} & \multicolumn{1}{l}{\textbf{QNLI}} & \multicolumn{1}{l}{\textbf{MRPC}} & \multicolumn{1}{l}{\textbf{Average}} \\
\midrule
LoRA            & 81.69                                 & 82.38                                  & 92.89                              & 82.17                             & 88.14                             & 86.27                             & 85.59                                \\
AdaLoRA         & 50.86                                 & 51.88                                  & 90.83                              & 79.19                             & 56.21                             & 82.35                             & 68.55                          \\
SoRA            & 72.17                                     &       74.32                               & 91.16                                  & 79.19                                 & 83.72                                 & 84.26                                 & 80.80                                    \\
DoRA            & 75.29                                 & 76.52                                  & 91.28                              & 79.77                             & 83.36                             & 83.82                             & 81.67                    \\
GoRA            & 81.37                                 & 81.78                                  & 91.97                              & 78.62                             & 86.49                             & \textbf{86.76}                    & 84.49                          \\
PARA            & \textbf{83.84}                        & \textbf{83.7}                          & \textbf{93.46}                     & \textbf{82.26}                    & \textbf{88.71}                    & \textbf{86.76}                    & \textbf{86.45}   \\   \bottomrule          
\end{tabular}
\end{table*}

\subsection{Commonsense Reasoning} \label{sec:commonsense}
To test commonsense reasoning, we adopt the Instruction Tuned Gemma3-4B \cite{gemma3} model as our frozen backbone. We train LoRA and its adaptive variants on the Commonsense-170k dataset \cite{commonsense170k} and test them across eight downstream tasks namely BoolQ \cite{boolq}, PIQA \cite{piqa}, SIQA \cite{siqa}, HellaSwag \cite{hellaswag}, WinoGrande \cite{winogrande}, ARC-C \cite{arc}, ARC-E \cite{arc} and OBQA \cite{obqa}. We set $r_{init} = 16$ and $r_{tgt} = 4$. Our empirical results demonstrate that PARA outperforms baselines in most benchmarks and attains the highest average accuracy. 

\begin{table*}[]
\centering
\caption{Accuracy Results of PARA and baselines with Gemma3-4B on Commonsense Reasoning Tasks. PARA outperforms baselines and scores highest (bold) or second-highest across datasets.}
\label{tab:commonsense-reasoning}
\begin{tabular}{@{}lccccccccc@{}}
\toprule
\textbf{Method} & \multicolumn{1}{l}{\textbf{BoolQ}} & \multicolumn{1}{l}{\textbf{PIQA}} & \multicolumn{1}{l}{\textbf{SIQA}} & \multicolumn{1}{l}{\textbf{HellaSwag}} & \multicolumn{1}{l}{\textbf{WinoGrande}} & \multicolumn{1}{l}{\textbf{ARC-C}} & \multicolumn{1}{l}{\textbf{ARC-E}} & \multicolumn{1}{l}{\textbf{OBQA}} & \multicolumn{1}{l}{\textbf{Average}} \\ \midrule
LoRA            & 83.6                               & 73.2                              & 49.8                              & 47.8                                   & 74                                      & 43.6                               & 54.4                               & 43.2                              & 58.7                                 \\
AdaLoRA         & 81.2                               & 70.3                              & 48.4                              & 46.5                                   & 71.6                                    & 43.4                               & 52.8                               & 42.4                              & 57.1                                 \\
SoRA            & 80.4                               & 69.1                              & 50.4                              & 45.9                                   & 70.2                                    & 43.2                               & 56.1                               & 41.8                              & 57.1                                 \\
DoRA            & 81.8                               & 68.2                              & \textbf{52.6 }                             & 47.3                                   & 70.2                                    & 42.8                               & \textbf{58.9}                               & 39.2                              & 57.6                                 \\
GoRA            & 83                                 & 67.6                              & 51.8                              & 46.8                                   & 70.4                                    & 43.2                               & 46.7                               & 40.4                              & 56.2                                 \\
PARA            & \textbf{84.4}                               & \textbf{74.8}                              & 50.2                              & \textbf{48.1}                                   & \textbf{74.2}                                    & \textbf{47.6  }                             & 58.2                               & \textbf{45.4 }                             & \textbf{60.4  }                               \\ \bottomrule
\end{tabular}
\end{table*}

\subsection{Mathematical Reasoning}
To test PARA and the baselines on Mathematical Reasoning tasks, we train on the MetaMathQA \cite{metamathqa} and test on GSM-8K \cite{gsm8k} and MATH \cite{math} benchmarks. Following Sec~\ref{sec:commonsense}, we use Gemma3-4B (IT) \cite{gemma3} as our backbone with $r_{init} = 16$ and $r_{tgt} = 4$. From Tab.\ref{tab:mathematical-reasoning}, we observe that PARA attains the highest performance in all benchmarks.

\begin{table}[] 
\centering
\caption{Accuracy Results of PARA and baselines with Gemma3-4B on Mathematical Reasoning Tasks. PARA outperforms baselines and scores highest (bold) across datasets.}
\label{tab:mathematical-reasoning}
\begin{tabular}{lrrr}
\toprule 
\textbf{Method} & \multicolumn{1}{l}{\textbf{GSM 8K}} & \multicolumn{1}{l}{\textbf{MATH}} &  \multicolumn{1}{l}{\textbf{Average}} \\ \midrule
LoRA            & 77.5                                & 44.1                              & 60.8                                 \\
AdaLoRA         & 70                                  & 42.3                              & 56.15                                \\
SoRA            & 71.2                                & 40.6                              & 55.9                                 \\
DoRA            & 72.4                                & 41.8                              & 57.1                                 \\
GoRA            & 76.5                                & 43.2                              & 59.85                                \\
PARA            & \textbf{78.2}                                & \textbf{45.5}                              & \textbf{61.8}                                                                 \\ \bottomrule
\end{tabular}
\end{table}

\begin{figure}
    \centering
    \includegraphics[width=0.9\linewidth]{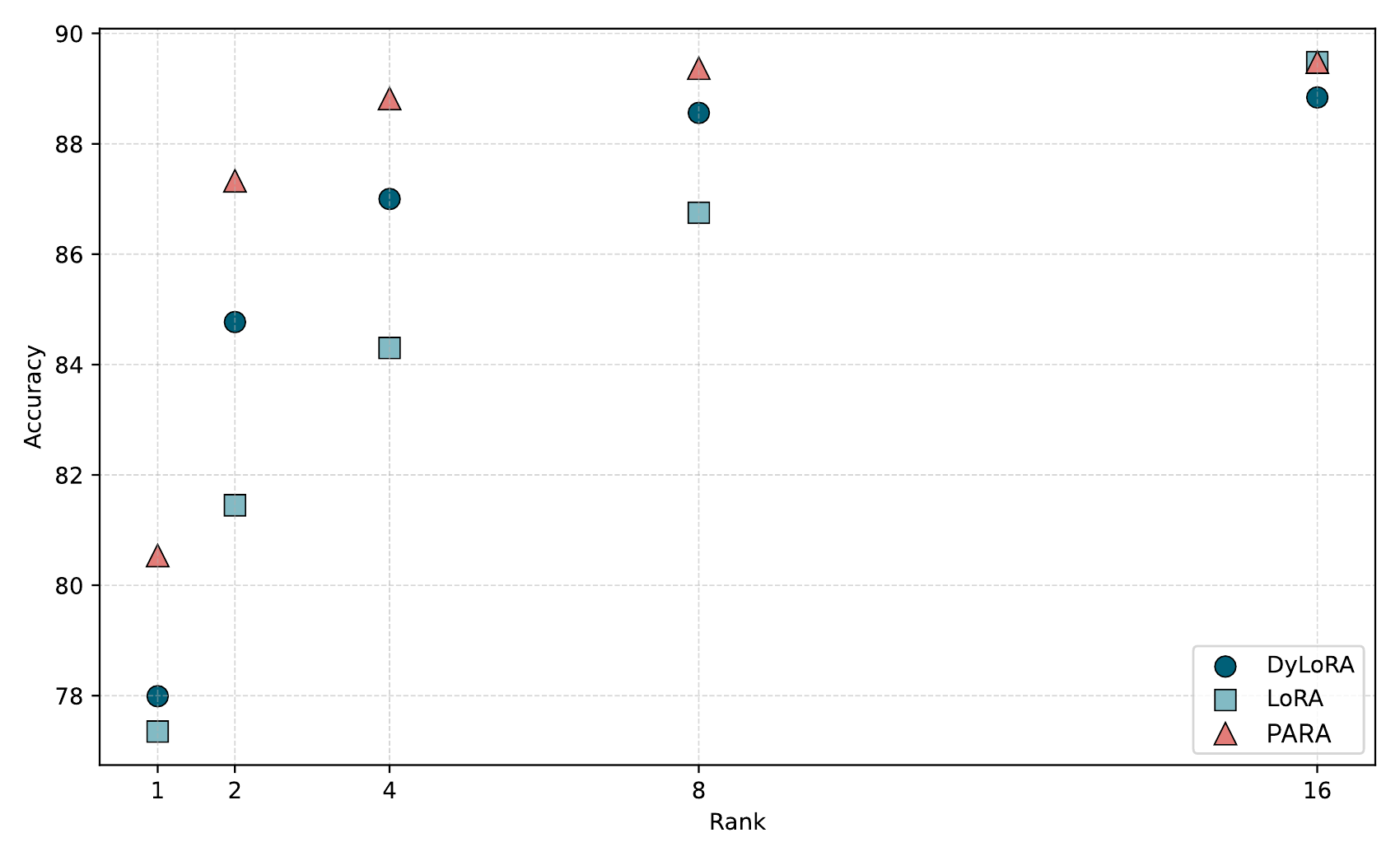}
    \caption{Scatter plot comparing PARA, DyLoRA and LoRA. Presented results are averaged across all image classification benchmarks as laid out in Sec.~\ref{sec:img-class}.}
    \label{fig:dylora-img}
\end{figure}

\subsection{Multi-Rank Deployment}
In multi-tenant serving systems, adapter swapping is a bandwidth bottleneck. Typically, supporting clients with different latency constraints requires training and storing multiple adapters at distinct ranks. We train a single high-rank LoRA and use PARA to generate a family of compressed LoRAs. This reduces storage costs and allows real-time serving adjustments based on current GPU load without retraining. DyLoRA \cite{dylora} simultaneously trains a range of ranks and enables dynamic rank selection at inference. Note that, unlike PARA which generates adaptive rank LoRAs, DyLoRA only supports standard, uniform-rank LoRAs. We compare LoRAs generated by PARA from a parent LoRA of rank 16 with DyLoRA trained in a range of ranks from 1 to 16 and LoRAs trained natively at ranks 1, 2, 4, 8 and 16. We present results from both the Image Classification (Fig.~\ref{fig:dylora-img}) and Natural Language Understanding (Fig.~\ref{fig:dylora-nlu}) benchmark suites.

\begin{figure}
    \centering
    \includegraphics[width=0.9\linewidth]{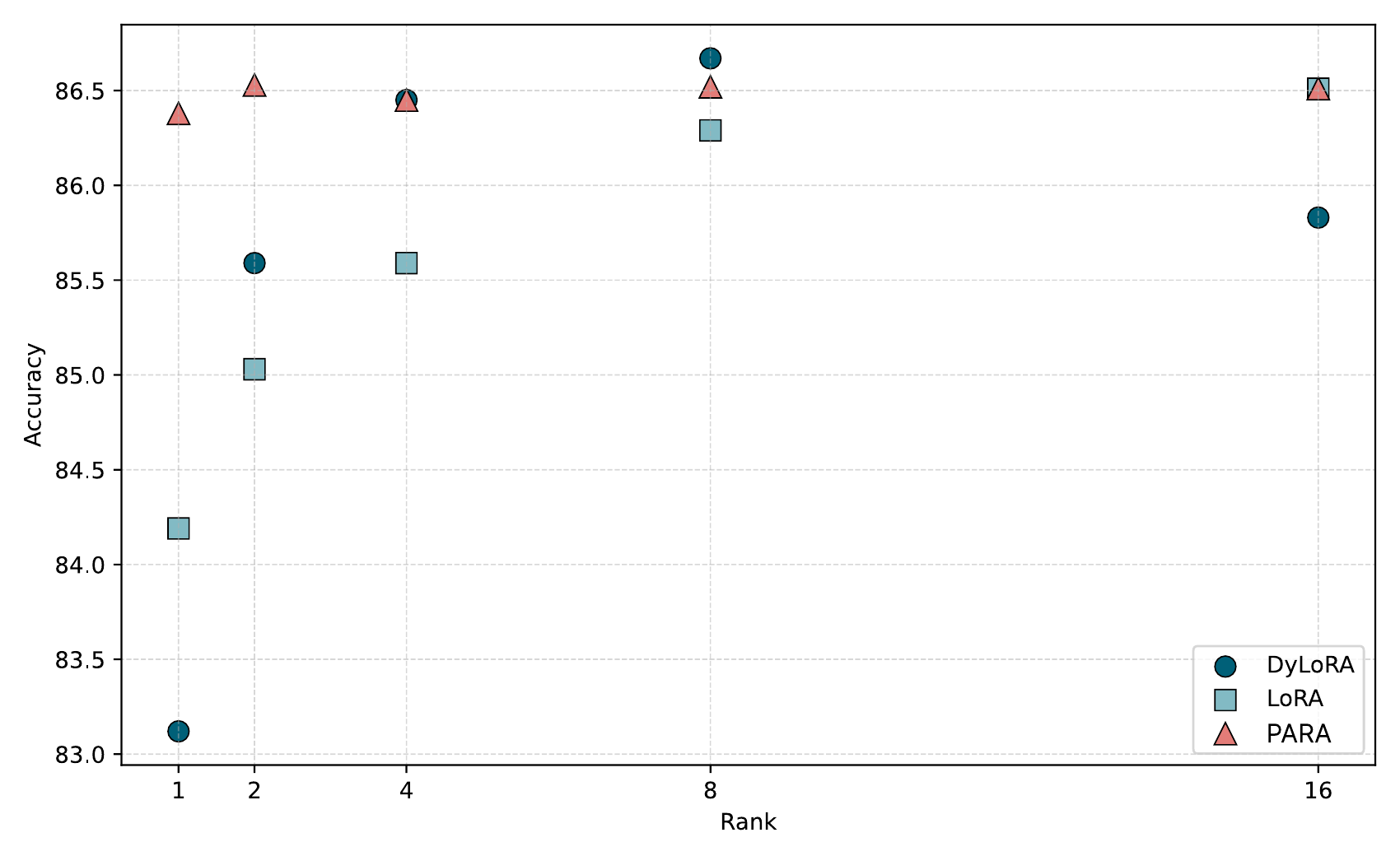}
    \caption{Scatter plot comparing PARA, DyLoRA and LoRA. Presented results are averaged across all natural language understanding benchmarks as laid out in Sec.~\ref{sec:nlu}}
    \label{fig:dylora-nlu}
\end{figure}

A consistent trend in both image classification and natural language understanding is that both DyLoRA and PARA-derived adapters often outperform LoRAs trained natively at the same target rank. We hypothesize that training a high-rank adapter provides additional degrees of freedom that make subspace discovery easier during optimization. The learned update typically concentrates most of its energy in a small number of dominant directions, so truncation recovers these task-relevant components while discarding low-energy variance. DyLoRA employs a similar idea during training via rank sampling and prefix-ordered components, which implicitly regularizes the adapter and encourages useful information to be represented in early rank dimensions, improving performance at low truncation ranks relative to independently trained low-rank LoRAs. Empirically, DyLoRA reports higher accuracy over LoRA at matched ranks, while PARA outperforms both native LoRAs and DyLoRA.

\subsection{Ablations}

\subsubsection{Fisher Importance}
Optimal compression necessitates the identification of parameters whose removal minimally impacts model performance. This sensitivity is formally quantified via the second-order Taylor expansion of the loss function, where the governing component is the Hessian matrix ($\mathbf{H}$) containing the second-order partial derivatives. Fundamentally, the Hessian measures the local ``curvature'' of the loss landscape: high curvature values indicate that the model resides in a steep valley with respect to a specific parameter, where even slight perturbations cause significant increases in loss, thereby deeming the parameter \textit{important}. Conversely, low curvature implies a flat plateau where parameters can be modified or removed with negligible effect, rendering them \textit{unimportant}. Since computing the exact full Hessian for billion-parameter models is computationally intractable, practitioners typically utilize the Empirical Fisher Information Matrix (FIM) as a positive semi-definite approximation.  Furthermore, to alleviate memory constraints, it is standard practice to ignore correlations between parameters by calculating only the diagonal of the Fisher matrix, effectively treating the importance of each parameter independently.

A central premise of PARA is that singular values serve as a sufficient proxy for parameter importance within the low-rank subspace, circumventing the need for computationally expensive Hessian or Fisher Information Matrix (FIM) approximations. To validate this, we compare PARA against a \textit{Fisher-PARA} baseline, where the importance score for a given rank is derived from the FIM of its constituent parameters. 

Ignoring layer indices for brevity, for a decomposed LoRA adapter $\phi = U \Sigma V^T$ where $U \in \mathbb{R}^{d_1 \times r}$, $\Sigma \in \mathbb{R}^{r \times r}$, and $V \in \mathbb{R}^{d_2 \times r}$, we define the Fisher Importance Score $\mathcal{I}_j$ for the $j$-th rank component as the sum of the empirical Fisher information of its corresponding singular vectors and singular value:
$$\mathcal{I}_j = \sum_{p=1}^{d_1} \hat{F}(U_{pj}) + \hat{F}(\Sigma_{jj}) + \sum_{q=1}^{d_2} \hat{F}(V_{qj})$$
where $\hat{F}(\omega)$ represents the diagonal of the empirical Fisher Information Matrix for a specific parameter $\omega$, calculated over a validation set $\mathcal{D}_{val}$:
$$\hat{F}(\omega) = \frac{1}{|\mathcal{D}_{val}|} \sum_{n=1}^{N} \left( \frac{\partial \mathcal{L}(x_n, y_n; \Theta)}{\partial \omega} \right)^2$$

As shown in Figure~\ref{fig:fisher-compare}, PARA achieves performance parity with Fisher-PARA across all benchmarks. Points lying near the $y=x$ diagonal indicate that our singular-value proxy captures the same essential importance signals as the gradient-based Fisher metric. Crucially, PARA eliminates the need for the 50-batch gradient computation required by Fisher, providing a truly data-free and significantly faster compression strategy. 

\subsubsection{Local Pruning}
PARA employs a global parameter budget $\mathcal{B}_{tgt}$, allowing rank to vary dynamically across layers based on spectral significance. We compare this against a \textit{Local Pruning} baseline, which enforces a uniform rank $r_{local} = \frac{\mathcal{B}_{tgt}}{\text{layers}}$ across the entire network. 

The Global policy consistently outperforms the Local policy, particularly at high compression ratios. We observe that PARA tends to allocate higher ranks to specific layers, suggesting that task-specific knowledge is not uniformly distributed. This affirms that a global spectral threshold effectively reallocates the parameter budget to where it is most needed. 

\begin{table}[t]
\caption{Results (Accuracy \%) comparing PARA (Global pruning) and Local pruning.}
\centering
\begin{tabular}{@{}lcccc@{}}
\toprule
\multicolumn{5}{l}{\textit{Image Classification}} \\ 
\midrule
              & \textbf{CIFAR 100} & \textbf{Food} & \textbf{Flowers} & \textbf{Cars} \\ 
Local Pruning & 78.23              & 84.49         & 85.15            & 83.82         \\
PARA          & 79.08              & 86.40         & 86.03            & 84.58         \\

\midrule
\multicolumn{5}{l}{\textit{Natural Language Understanding}} \\ 
\midrule
              & \textbf{QNLI} & \textbf{MRPC} & \textbf{CoLA} & \textbf{SST-2} \\ 
Local Pruning & 86.85              & 83.17         & 81.64            & 92.22         \\
PARA          & 88.71              & 86.76         & 82.26            & 93.46         \\
\bottomrule
\end{tabular}
\end
{table}

\begin{figure}
    \centering
    \includegraphics[width=0.90\linewidth]{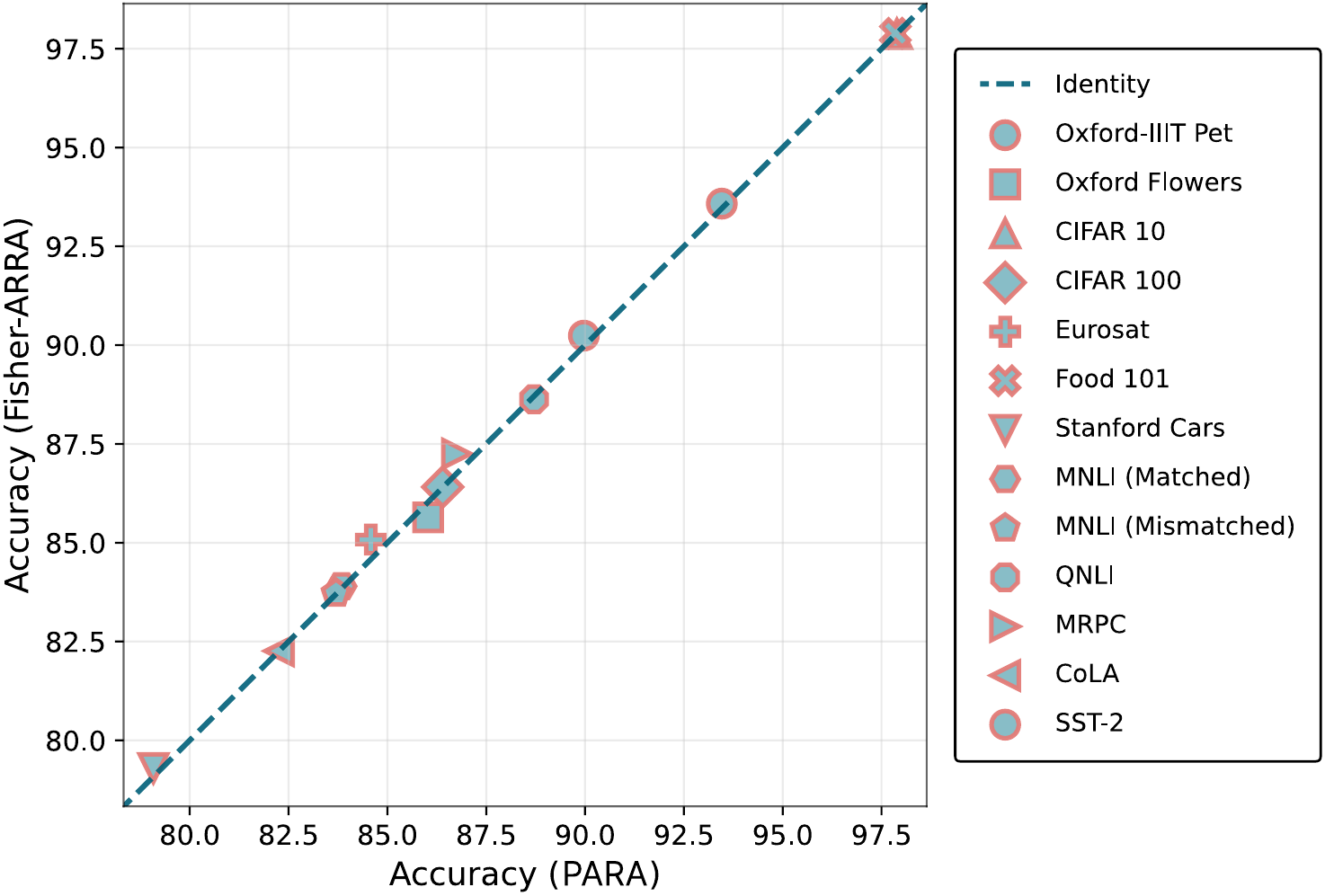}
    \caption{Scatter plot comparing PARA and Fisher-PARA accuracies across image classification and natural language understanding datasets.}
    \label{fig:fisher-compare}
\end{figure}

\section{Conclusion}

In this work, we introduce the Post-Optimization Adaptive Rank Allocation (PARA), a post-optimization compression framework that decouples the training rank from the inference rank without the need for complex regularization or training instability. PARA efficiently identifies and prunes redundant spectral components, serving as a data-free proxy for parameter importance that matches the efficacy of Fisher Information metrics. Our empirical results across diverse vision and language benchmarks demonstrate that PARA consistently outperforms existing adaptive baselines while enabling a practical "Train First, Tune Later" paradigm. Ultimately, PARA provides a robust solution for multi-tenant serving environments, allowing practitioners to derive a family of lightweight, hardware-compliant adapters from a single high-capacity parent adapter.

\clearpage

\bibliography{main}


\clearpage

\appendix
\section*{Appendix}

\begin{figure*}
    \centering
    \includegraphics[width=0.99\linewidth]{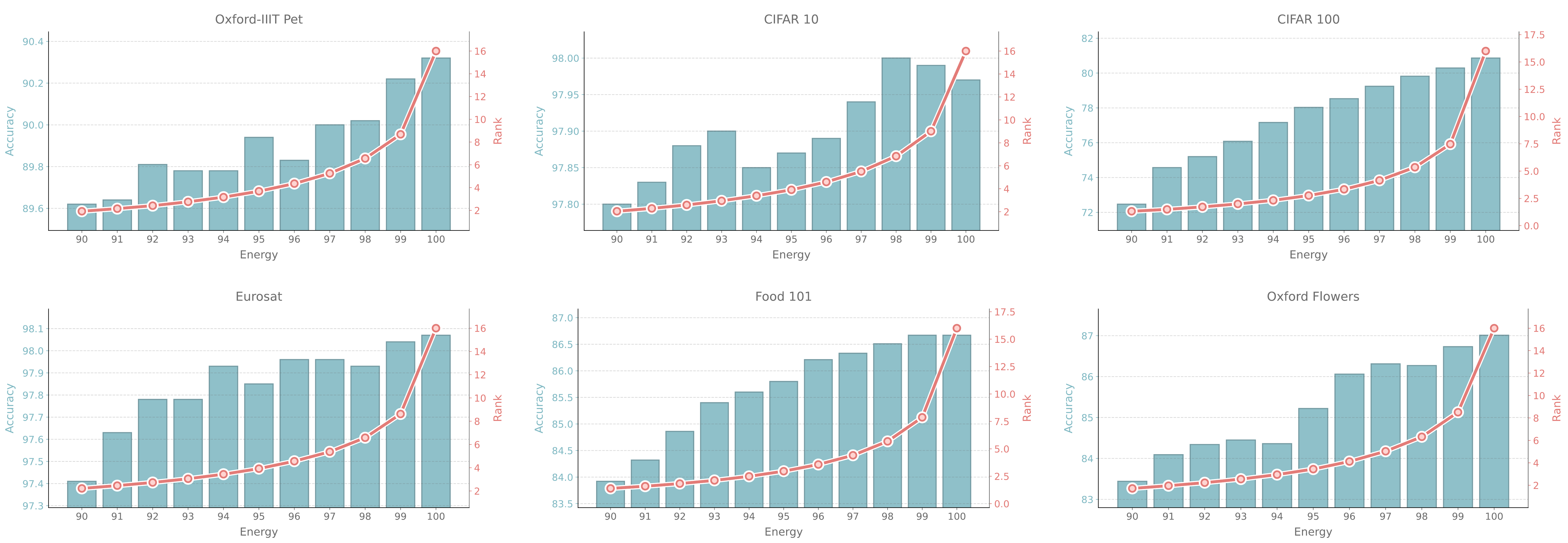}
    \caption{Accuracy vs rank during Energy based Compression on Image Classification benchmarks.}
    \label{fig:appendix-dual}
\end{figure*}

\begin{figure*}
    \centering
    \includegraphics[width=0.99\linewidth]{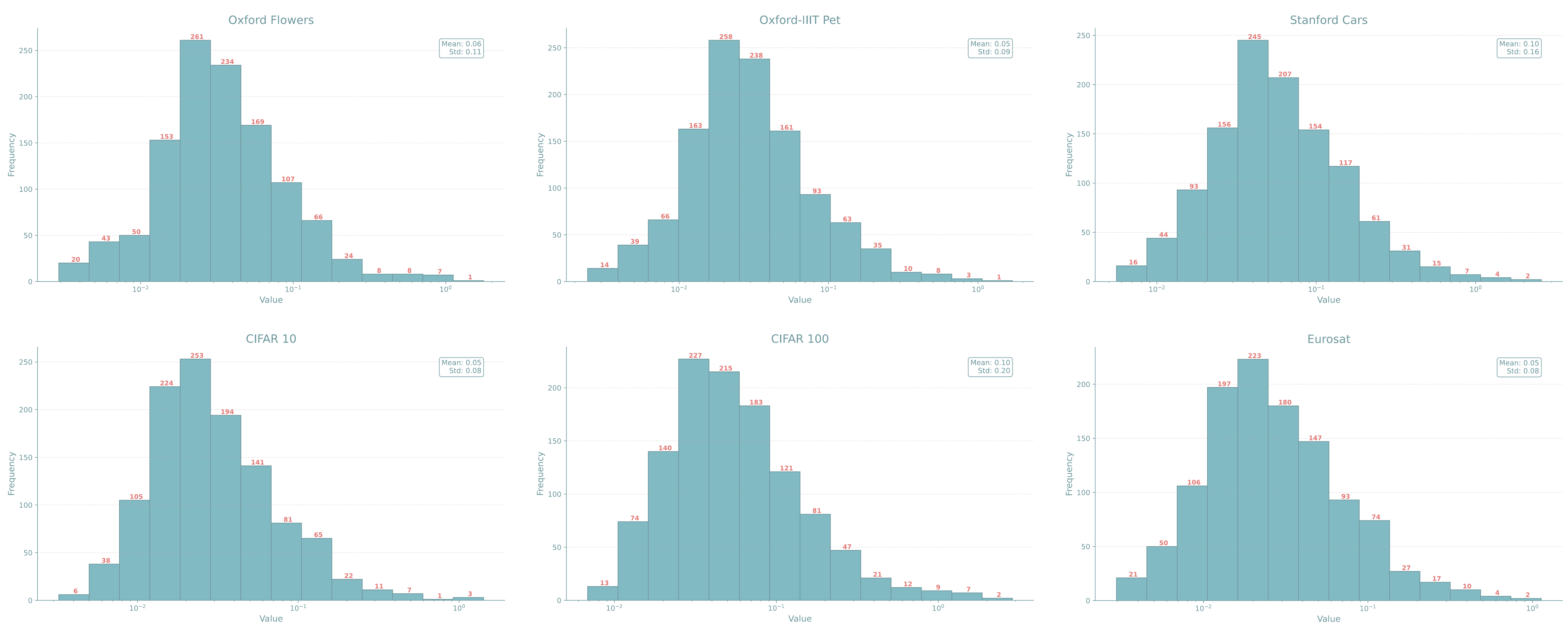}
    \caption{Distribution of singular values in LoRAs trained on different image classification datasets.}
    \label{fig:appendix-hist}
\end{figure*}

\section{Implementation Details}

We compare against AdaLoRA, DoRA, DyLoRA, SoRA, and GoRA using their standard settings: AdaLoRA and DoRA start from rank 16 and adapt toward rank 4 with a scheduled transition (\(t_{\text{init}}=500\), \(t_{\text{final}}=500\), \(\Delta T=500\)); DoRA additionally uses \(\eta_{\mathrm{DEM}}=0.3\) and \(\beta=0.9\), while AdaLoRA uses an orthogonality regularizer of \(10^{-5}\). DyLoRA spans ranks from 1 to 16 and evaluates \(\{1,2,4,8,16\}\). SoRA uses \(\lambda=0.1\), \(\texttt{gate\_lr}=10^{-3}\), \(\texttt{warmup}=500\), and \(\texttt{target\_sparsity}=0.7\). GoRA is configured with \(\texttt{ref\_r}=4\), \(\texttt{init\_batches}=64\), and \(\gamma=0.05\). Unless noted (e.g., SoRA gating), all baselines use the same optimizer and learning rate as the main image-classification setup (\(\text{AdamW},\, 5\times 10^{-5}\)).

\end{document}